\documentclass{article} %
\usepackage{iclr2027_conference,times}

\usepackage{amsmath,amsfonts,bm}

\def\eqref#1{equation~\ref{#1}}

\def\1{\bm{1}}

\DeclareMathAlphabet{\mathsfit}{\encodingdefault}{\sfdefault}{m}{sl}
\SetMathAlphabet{\mathsfit}{bold}{\encodingdefault}{\sfdefault}{bx}{n}

\usepackage{hyperref}
\usepackage{url}

\usepackage{amsmath}        %
\usepackage{amssymb}        %

\usepackage{amsthm}         %
\usepackage{multirow}       %
\theoremstyle{plain}
\newtheorem{theorem}{Theorem}

\newtheorem{proposition}{Proposition}
\newtheorem{lemma}{Lemma}
\usepackage{wrapfig}
\usepackage{enumitem} %
\usepackage{simpleicons}

\newif\ifarxiv
\arxivtrue
\usepackage{graphicx}       %
\usepackage{algorithm}      %
\usepackage{algorithmic}    %
\usepackage{tabularx}       %
\usepackage{booktabs}       %
\usepackage{nicefrac}       %
\usepackage[table]{xcolor}  %

\usepackage{subcaption}

\title{Beam Search as Test-Time Self-Distillation via Counterfactual Contexts}

\author{Su Ee Tan \\
\textnormal{\texttt{stan.suee@gmail.com}}
\And
Xiaotong Ji \\
\textnormal{Huawei Noah’s Ark Lab }
\And 
Rasul Tutunov \\
\textnormal{Huawei Noah’s Ark Lab}
\And
Haitham Bou-Ammar \\
\textnormal{Huawei Noah’s Ark Lab}\\
\textnormal{UCL Centre for AI}
\And
Matthieu Zimmer \\
\textnormal{Huawei Noah’s Ark Lab}
}

\ifarxiv
\iclrfinalcopy
\fi

\begin{document}

\maketitle

\ifarxiv
\lhead{Preprint}

\fi

\begin{abstract}
Self-Distillation Fine-Tuning (SDFT) enables a language model to act as its own teacher: by conditioning on a demonstration, the model produces an implicit reward via pointwise mutual information, which guides on-policy learning without external supervision.
However, SDFT operates at \emph{training time}: it requires gradient updates and access to expert demonstrations, making it inapplicable at inference.
We propose \emph{test-time self-distillation}, a decoding-time method that extracts a steering signal from the self-distillation framework without any parameter updates, reward models, or training data.
Our key insight is that \emph{counterfactual contexts}, i.e. fixed textual templates that hypothetically prime the model for excellent versus poor reasoning, can substitute for the demonstration.
The log-odds ratio of a candidate answer under these two counterfactual conditions defines a new reward signal. We derive the optimal KL-regularized policy under this reward, which takes the form of a Gibbs reweighting of the base distribution.
Crucially, this reweighting is global: it cannot be decomposed into independent per-token operations without ignoring future trajectory quality.
We therefore approximate the target distribution via beam search.
Experiments on mathematical reasoning (MATH500), code generation (HumanEval), and graduate-level science QA (GPQA) across multiple model scales show that test-time self-distillation improves over standard sampling, low temperature, beam search and power sampling baselines on average, demonstrating that the self-distillation principle can be operationalized at inference time.
\end{abstract}

\section{Introduction}                                                                 
During training, language models can serve as their own teachers through Self-Distillation Fine-Tuning (SDFT) \citet{shenfeld2026selfdistillationenablescontinuallearning,hubotter2026reinforcementlearningselfdistillation}. The model can be conditioned on an expert demonstration $c$ and acts as a teacher for the same model without the demonstration.
The key object is the \emph{pointwise mutual information} (PMI) between the demonstration and the model's output, which serves as an implicit reward: $r(a, q, c) = \log \pi(a \mid q, c) - \log \pi(a \mid q)$.
By optimizing this reward under a KL constraint via on-policy distillation, SDFT eliminates the need for external task-specific rewards.
However, SDFT is fundamentally a \emph{training-time} method. It requires (i) an expert demonstration $c$ for each task, (ii) gradient updates to the model parameters, and (iii) an on-policy training loop that repeatedly samples from the student and computes loss against the teacher.
At \emph{test time}, when a user poses a question and the model must generate an answer, none of these components are available. The model is frozen, there is no demonstration, and there is no training loop.
This raises a natural question: \emph{can the self-distillation principle be operationalized at test time, without any parameter updates?}

We propose a solution to this question by replacing the expert demonstration with \emph{counterfactual contexts}.
Rather than conditioning on a real demonstration $c$, which is not available at test time, we condition the model on two fixed textual templates: a positive context $c^+$ that hypothetically primes the model for excellent reasoning (e.g., \textit{``This is an example for a response to the question with excellent reasoning:''}) and a negative context $c^-$ that primes it for poor reasoning (e.g., \textit{``This is an example for a response to the question with wrong reasoning:''}).
These contexts are never shown to the user and do not appear in the generated output.
They are \emph{counterfactual probes} that ask: \emph{how would the model's probability of this answer change if it were reasoning excellently versus poorly?}

The contrastive PMI between these counterfactual contexts defines an implicit reward:
\begin{equation}
    R(q, a) = \log \pi_\theta(a \mid q, c^+) - \log \pi_\theta(a \mid q, c^-).
    \label{eq:reward}
\end{equation}
Intuitively, an answer that is more likely under the excellent-reasoning counterfactual than under the poor-reasoning counterfactual receives a positive reward, steering generation toward responses the model associates with high-quality reasoning.

This contrastive formulation is necessitated by the absence of ground-truth supervision at test time. In SDFT, an expert demonstration $c$ is available, the reward $r(y, x, c) = \log \pi(y \mid x, c) - \log \pi(y \mid x)$ measures absolute alignment with $c$. Lacking any such oracle, we resort to a preference-based approach: rather than measuring absolute quality, we elicit the model's relative assessment under two counterfactual conditions.
Connecting this to the framework of \citet{rafailov2024directpreferenceoptimizationlanguage}, the reward $R(q, a)$ can be interpreted as a Bradley-Terry preference signal: the probability that the answer $a$ is ``better'' under $c^+$ than under $c^-$ is $\sigma(R(q, a) / \beta)$, where $\beta$ is a temperature parameter.

Maximizing this reward subject to a KL-divergence constraint yields a closed-form target distribution:
\begin{equation}
    \pi_{\text{target}}(a \mid q) \propto \pi_\theta(a \mid q) \exp\!\left(\frac{1}{\alpha} R(q, a)\right) = \pi_\theta(a \mid q) \left(\frac{\pi_\theta(a \mid q, c^+)}{\pi_\theta(a \mid q, c^-)}\right)^{\!1/\alpha},
    \label{eq:target}
\end{equation}
where $\alpha$ controls the steering strength. This is a Gibbs reweighting of the base policy, tilted toward answers favored under the excellent-reasoning counterfactual.

Sampling from $\pi_{\text{target}}$ exactly is intractable due to the global normalization constant $Z(q)$.
However, as \citet{ji2026scalablepowersamplingunlocking, nguyen2026modelknowsdecoderfinds} demonstrate in the context of power distribution sampling \cite{karan2025reasoningsamplingbasemodel}, local token-level approximations ignore the future: the per-token factorization $\prod_t \pi(a_t \mid q, a_{<t}) \exp(r_t / \alpha)$ drops the global normalization $Z(q)$, which depends on the entire trajectory. This is analogous to the gap between \emph{low-temperature sampling} (a local operation that sharpens each token independently) and the \emph{power distribution} $p^\alpha$ (a global reweighting that accounts for future trajectory quality). In fact, the power distribution is a special case where the reward is self-reinforcement, $R(a, q) = (\alpha - 1)\log \pi_\theta(a \mid q)$; our contrastive PMI reward $R(a, q) = \log \pi_\theta(a \mid q, c^+) - \log \pi_\theta(a \mid q, c^-)$ extends it to counterfactual contexts.

Because the target distribution is a global reweighting, we approximate it via \emph{beam search} as done by \citet{ji2026scalablepowersamplingunlocking}.
This approach evaluates the reward on partial trajectories. It requires two additional forward passes per beam compared to traditional beam search, however it does not introduce new parameters, training data, or external reward models.

We evaluate test-time self-distillation on mathematical reasoning (MATH500), code generation (HumanEval), and graduate-level science QA (GPQA) using Qwen3.5-0.8B, Qwen2.5-7B, Qwen3.5-2B and Deepseek-Math-7B-Instruct.
Across these settings, the method improves accuracy over all baselines (vanilla sampling, low-temperature sampling, beam search and power sampling).
These results suggest that the self-distillation principle can be recovered at inference time.

\paragraph{Contributions.} We make the following contributions:
\begin{itemize}[leftmargin=1.5em, topsep=0pt, itemsep=0pt]
    \item We introduce \emph{test-time self-distillation}, a framework that extends the SDFT self-distillation principle to inference time without parameter updates, external reward models, or training data.
    \item We show that sampling from this target distribution cannot be decomposed into independent per-token operations: exact per-token sampling requires a \emph{future correction} term (Proposition~\ref{thm:future_correction}) that is as intractable as the global normalization constant. 
    \item We propose \emph{contrastive beam search} and analyze that the truncation error of this approximation is bounded and vanishes as generation progresses (Theorem~\ref{thm:truncation_bias}).
    Across four models and three benchmarks, it outperforms standard sampling, low-temperature decoding, beam search, and power sampling, with ablations confirming that the semantic contrast drives the gains.
\end{itemize}

\section{Background}
\subsection{KL-Regularized Reward Maximization}
\label{sec:bg_kl}

A central framework in language model alignment is the maximization of an expected reward subject to a KL-divergence constraint \citep{Korbak2022RLWK,rafailov2024directpreferenceoptimizationlanguage}. Given a reward function $r(a, q)$ that evaluates the action $a$ for a question $q$ and a reference policy $\pi_k(\cdot \mid q)$, the optimization problem is:
\begin{equation}
    \pi^* = \arg\max_\pi \; \mathbb{E}_{a \sim \pi(\cdot \mid q)}\!\left[r(a, q)\right] - \beta \, D_{\mathrm{KL}}\!\left(\pi(\cdot \mid q) \,\|\, \pi_k(\cdot \mid q)\right),
    \label{eq:kl_reg}
\end{equation}
where $\beta > 0$ controls the strength of the KL penalty. This objective has a well-known closed-form solution:
\begin{equation}
    \pi^*(a \mid q) = \frac{1}{Z(q)} \, \pi_k(a \mid q) \, \exp\!\left(\frac{1}{\beta}\, r(a, q)\right),
    \label{eq:tilted}
\end{equation}
where $Z(q) = \sum_a \pi_k(a \mid q) \exp(r(a,q) / \beta)$ is a normalization constant. This \emph{tilted} or \emph{Gibbs} distribution reweights the reference policy by the exponential of the reward, concentrating probability mass on high-reward outputs while remaining anchored to the base distribution.

\subsection{Self-Distillation Fine-Tuning}
\label{sec:bg_sdft}

Self-Distillation Fine-Tuning (SDFT) \citep{shenfeld2026selfdistillationenablescontinuallearning} extends the on-policy distillation framework \citep{hinton2015distillingknowledgeneuralnetwork,xiong2024distillingmorphologyconditionedhypernetworksefficient} by removing external reward. Given a base model with policy $\pi_\theta$ and an expert demonstration $c$ for a task with prompt $q$, SDFT constructs a \emph{teacher} by conditioning the same model on the demonstration: $\pi(\cdot \mid q, c)$. The \emph{student} is the base model $\pi_\theta(\cdot \mid q)$. Training aims at minimizing the reverse KL divergence between student and teacher:
$\mathcal{L}(\theta) = \mathbb{E}_{a \sim \pi_\theta(\cdot \mid q)}\!\left[\log \frac{\pi_\theta(a \mid q)}{\pi(a \mid q, c)}\right].$

The key insight of SDFT is that this distillation objective is equivalent to policy gradient under an implicit reward derived from the \emph{pointwise mutual information} between the demonstration and the model's output. Substituting the teacher $\pi(\cdot \mid q, c)$ as the optimal policy $\pi^*$ in the KL-regularized framework Eq.~\ref{eq:kl_reg} yields the reward:
$r(a, q, c) = \log \pi(a \mid q, c) - \log \pi(a \mid q).$
This PMI reward measures how much the demonstration $c$ ``explains'' the output $a$: outputs whose probability increases under the demonstration receive positive reward.

SDFT requires three components that are unavailable at test time: (i)~an expert demonstration $c$, (ii)~gradient updates to $\theta$, and (iii)~an on-policy training loop. 
Our work replaces the demonstration with counterfactual contexts and replaces gradient-based optimization with decoding-time approximation, preserving the PMI reward structure while eliminating all training-time requirements.

\section{Method}
We now formalize test-time self-distillation. The method has three components: (1)~counterfactual contexts that define a contrastive reward (Section~\ref{sec:method_contexts}), (2)~a theoretical analysis showing that naive per-token decoding fails due to a future correction term (Section~\ref{sec:method_theory}), and (3)~a contrastive beam search algorithm that approximates the target distribution (Section~\ref{sec:method_beamsearch}).

\subsection{Counterfactual Contexts and Per-Token Reward}
\label{sec:method_contexts}

We replace the expert demonstration $c$ from SDFT with two question-independent fixed textual templates $c^+$ and $c^-$.
These are appended to the question $q$ when computing the model's log-probabilities. We propose to use the following contrastive PMI reward $R(q, a) = \log \pi_\theta(a \mid q, c^+) - \log \pi_\theta(a \mid q, c^-)$ (Eq.~\ref{eq:reward}), it decomposes via the autoregressive factorization into per-token rewards:
\begin{equation}
    R(a, q) = \sum_{t=1}^{T} r_t(a_t, q), \quad r_t(a_t, q) = \log \pi_\theta(a_t \mid q, a_{<t}, c^+) - \log \pi_\theta(a_t \mid q, a_{<t}, c^-),
    \label{eq:per_token_reward}
\end{equation}
where $r_t(a_t, q)$ measures the change in log-odds for token $a_t$ under the two counterfactual conditions given prefix $a_{<t}$. The KL-regularized optimal policy under this reward is the Gibbs reweighting $\pi_{\text{target}}(a \mid q) \propto \pi_\theta(a \mid q) \exp(R(a, q)/\alpha)$ (Eq.~\ref{eq:target}). The central challenge is sampling from this distribution.

\subsection{The Future Correction}
\label{sec:method_theory}

A naive approach to sampling from $\pi_{\text{target}}$ would greedily follow the locally reweighted distribution
\begin{equation}
    \tilde{\pi}(a_t \mid q, a_{<t}) \propto \pi_\theta(a_t \mid q, a_{<t}) \exp\!\left(\frac{1}{\alpha} r_t(a_t, q)\right).
    \label{eq:naive}
\end{equation}
This applies the contrastive reweighting independently at each token, analogous to how low-temperature sampling applies the power transformation locally \citep{ji2026scalablepowersamplingunlocking}. 
The following proposition shows that exact per-token sampling from $\pi_{\text{target}}$ instead requires a \emph{future correction} $\zeta_t$ accounting for downstream reward.

\begin{proposition}[Per-Token Decomposition with Future Correction]
\label{thm:future_correction}
Let $\pi_\theta$ be an autoregressive language model over vocabulary $\mathcal{V}$, and let the target distribution be
\begin{equation*}
    \pi_{\mathrm{target}}(a \mid q) = \frac{1}{Z(q)} \prod_{t=1}^{T} \pi_\theta(a_t \mid q, a_{<t}) \exp\!\left(\frac{1}{\alpha} r_t(a_t, q)\right),
\end{equation*}
where $r_t(a_t, q) = \log \pi_\theta(a_t \mid q, a_{<t}, c^+) - \log \pi_\theta(a_t \mid q, a_{<t}, c^-)$. Then for any partial sequence $a_{<t}$, the per-token conditional of the target distribution is
\begin{equation}
    \pi_{\mathrm{target}}(a_t \mid q, a_{<t}) = \frac{\pi_\theta(a_t \mid q, a_{<t}) \exp\!\left(\frac{1}{\alpha} r_t(a_t, q)\right) \zeta_t(a_t, q)}{\sum_{a' \in \mathcal{V}} \pi_\theta(a' \mid q, a_{<t}) \exp\!\left(\frac{1}{\alpha} r_t(a', q)\right) \zeta_t(a', q)},
    \label{eq:per_token_decomp}
\end{equation}
where the \emph{future correction} is
\begin{equation}
    \zeta_t(a_t, q) = \sum_{a_{t+1:T}} \prod_{s=t+1}^{T} \pi_\theta(a_s \mid q, a_{<t}, a_t, a_{<s}) \exp\!\left(\frac{1}{\alpha} r_s(a_s, q)\right).
    \label{eq:future_correction}
\end{equation}
\end{proposition}

The proof is available in Appendix~\ref{app:proofs}.
The naive per-token distribution (Eq.~\ref{eq:naive}) corresponds to setting $\zeta_t \equiv 1$, which is exact only at the final token ($t = T$). At all earlier positions, it ignores how each token choice affects the reward of future completions, i.e. the same local--global gap that separates low-temperature sampling from the power distribution. Computing $\zeta_t$ exactly requires marginalizing over all future trajectories, which is as intractable as computing $Z(q)$.

\subsection{Contrastive Beam Search (CBS)}
\label{sec:method_beamsearch}

Since the target distribution is a global reweighting, we approximate it via block-wise beam search with contrastive scoring. The search operates over chunks of $K$ tokens rather than individual tokens, alternating between stochastic extension and contrastive pruning. At each iteration, $N$ active beams are extended by one chunk, scored under three prompt contexts, and pruned to the top $N/W$ beams, which are then replicated to restore the population. Completed trajectories are moved to a terminal pool. See Algorithm~\ref{alg:contrastive_beam} for the full details.

\begin{algorithm}[t]
\caption{Contrastive Beam Search}
\label{alg:contrastive_beam}
\begin{algorithmic}[1]
\REQUIRE Question $q$, model $\pi_\theta$, contexts $c^+, c^-$
\REQUIRE Population $N$, pruning factor $W$, block size $K$, iterations $T$, temperature $\tau$, steering strength $\alpha$
\ENSURE Selected response $\hat{a}$
\STATE Initialise $\mathcal{B} \leftarrow \{\}^{N}$; $\mathcal{A} \leftarrow \emptyset$; $t=0$
\WHILE{$t < T$}
    \STATE Set a temporary buffer:  $\mathcal{C} \leftarrow \emptyset$
    \FORALL{$a \in \mathcal{B}$}
        \STATE Sample $\delta \sim \pi_\theta(\cdot \mid q, a;\, \tau)$ with temperature $\tau$ at most $K$ tokens
        \STATE Concatenate: $a' \leftarrow [a,\delta]$ and compute $s(a', q)$ using Equation \ref{eq:beam_score}
        \IF{$a'$ terminates or $t = T$}
            \STATE $\mathcal{A} \leftarrow \mathcal{A} \cup \{(a', s(a',q))\}$
        \ELSE
            \STATE $\mathcal{C} \leftarrow \mathcal{C} \cup \{(a', s(a', q))\}$
        \ENDIF
    \ENDFOR
    \STATE Remove duplicate trajectories from $\mathcal{C}$
    \IF{$\mathcal{C} = \emptyset$} \STATE \textbf{break} \ENDIF
    \STATE $\mathcal{S} \leftarrow$ top $\min(N/W,\, |\mathcal{C}|)$ elements by highest score from $\mathcal{C}$
    \STATE $\mathcal{B} \leftarrow$ replicate $\mathcal{S}$ to $N$ beams
    \STATE Increment the counter $t = t+1$
\ENDWHILE
\STATE \RETURN $\hat{a} \leftarrow \arg\max_{(a, s) \in \mathcal{A}} s(a,q)$
\end{algorithmic}
\end{algorithm}

The score $s(a,q)$ of a beam composed of the tokens $a$ approximates the log of the unnormalized target distribution:
\begin{equation}\label{eq:beam_score}
\frac{1}{|a|} \sum_{t=1}^{|a|} \left[ \log \pi_\theta(a_t \mid q, a_{<t}) + \frac{1}{\alpha} \left( \log \pi_\theta(a_t \mid q, c^+, a_{<t}) - \log \pi_\theta(a_t \mid q, c^-, a_{<t}) \right) \right].
\end{equation}
The contrastive difference inside the inner parentheses corresponds to the per-token reward $r_t$ (Eq.~\ref{eq:per_token_reward}).
We use the mean rather than the sum to avoid length bias across beams of different lengths. Each beam requires two forward passes per chunk (base log probabilities can be reused from the generation process), so the total cost is slightly larger than standard beam search. However, this additional cost remains manageable because generation is usually more expensive than batched prefill operations in GPUs.

The beam score accumulates the contrastive reward over all tokens generated so far, which is equivalent to setting the future correction $\hat{\zeta}_t \equiv 1$ at each scoring point. Unlike naive per-token sampling (Eq.~\ref{eq:naive}), which truncates at every token position with horizon $T - t$, beam search generates $K$-token chunks and scores them exactly, so the truncation only affects the future \emph{beyond} each chunk (horizon $T - t - K$). The following theorem bounds the resulting approximation error.

\begin{theorem}[Truncation Bias of Contrastive Beam Search]
\label{thm:truncation_bias}
Assume per-token rewards are bounded: $|r_t(a_t, q)| \leq r_{\max}$ for all $t, a_t$, and let $K$ be the chunk size. At a scoring point at position $t$ (start of a chunk), the $K$ tokens within the chunk are scored exactly; the future correction $\zeta_t$ is truncated for the remaining $T - t - K$ tokens by setting $\hat{\zeta}_t \equiv 1$. The per-step TV distance between the truncated distribution $\tilde{\pi}_{\mathrm{trunc}}$ and the target $\pi_{\mathrm{target}}$ satisfies
\begin{equation}
    \left\|\tilde{\pi}_{\mathrm{trunc}}(\cdot \mid q, a_{<t}) - \pi_{\mathrm{target}}(\cdot \mid q, a_{<t})\right\|_{\mathrm{TV}} \leq \tanh\!\left(\frac{(T - t - K)\, r_{\max}}{\alpha}\right),
    \label{eq:truncation_bias}
\end{equation}
and the trajectory-level error compounds over $T/K$ scoring decisions to at most $\frac{r_{\max}}{\alpha} \cdot \frac{T(T-K)}{2K}$. Naive per-token sampling (Eq.~\ref{eq:naive}) truncates at every token with no chunk lookahead, giving the looser bound $\frac{r_{\max}}{\alpha} \cdot \frac{T(T-1)}{2}$; the chunk size $K$ yields an approximately $K$-fold reduction. Under the balanced-contexts condition $\mathbb{E}_{a_t \sim \pi_\theta}[r_t(a_t, q)] = 0$ (equivalently, the KL from the base model to $c^+$ equals the KL to $c^-$), the per-step bound improves to $\tanh\!\bigl(\frac{(T-t-K)\,r_{\max}^2}{4\alpha^2}\bigr)$ and the trajectory bound to $\frac{r_{\max}^2}{4\alpha^2} \cdot \frac{T(T-K)}{2K}$. In both cases, the per-step error vanishes as $t \to T - K$.
\end{theorem}

The proof is in Appendix~\ref{app:proofs}. Theorem~\ref{thm:truncation_bias} shows that beam search's truncation error is bounded and, crucially, decreases as generation progresses: early chunks (small $t$) incur the largest error because they ignore the most future reward, while later chunks are scored nearly exactly. The chunk size $K$ provides two advantages over naive per-token truncation: (i) the effective horizon at each scoring point is $K$ tokens shorter ($T - t - K$ vs.\ $T - t$), and (ii) scoring decisions are made $T/K$ times rather than $T$ times. Together these yield an approximately $K$-fold reduction in trajectory error, formalizing the advantage of chunk-level scoring.
To further reduce bias, it is also possible to use lookahead Monte-Carlo rollouts at the cost of greater inference time \citep{ji2026scalablepowersamplingunlocking}.

\section{Experiments}

We evaluate our Contrastive Beam Search (CBS) algorithm on four models and three benchmarks to examine if counterfactual contrastive scoring provides a reliable test-time steering signal across different models and reasoning tasks. CBS consistently improves over standard sampling and search baselines and achieves the best performance in most settings. We further analyze its inference cost and the beam-search design choices that contribute to these gains.

\subsection{Performance Evaluation}

\textbf{Models.} We test four base models across different model scales and families to assess the robustness and generality of the CBS algorithm: Qwen3.5-0.8B, Qwen3.5-2B \citep{qwen3.5}, Qwen2.5-7B \citep{qwen2024}, and DeepSeek-Math-7B-Instruct \citep{shao2024}. All models are evaluated in their original pretrained or fine-tuned forms, without any additional fine-tuning.

\textbf{Benchmarks.} We evaluate on three benchmarks covering distinct task types. For mathematics, we use MATH500 \citep{lightman2023}, which requires multi-step quantitative reasoning. For code, we use HumanEval \citep{chen2021}, a set of Python programming tasks scored by unit tests. For question answering, we use GPQA \citep{rein2024}, which covers graduate-level questions in biology, physics, and chemistry.

\textbf{Baselines.} We compare CBS against standard autoregressive sampling, low-temperature sampling, and beam search in our main results. Standard and low-temperature sampling use temperatures of $\tau=1.0$ and $\tau=0.25$, respectively. The beam search baseline \citep{snell2024} ranks partial trajectories using only the base-model log-likelihood.
\begin{table}[ht]
\centering
\begin{tabular}{llccc}
\toprule
\textbf{Model} & \textbf{Method} & \textbf{MATH500} & \textbf{HumanEval} & \textbf{GPQA} \\
\midrule
\multirow{7}{*}{\textbf{Qwen3.5-0.8B}}
& Base $\tau=1.0$ & 0.198 & 0.085 & 0.020 \\
& Low Temperature $\tau=0.25$ & 0.342 & 0.225 & 0.101 \\
& Beam Search & 0.398 & 0.207 & 0.091 \\
& Power Sampling $\tau=0.25$ & 0.430 & 0.287 & \textbf{0.131} \\
& Power Sampling $\tau=1.0$ & 0.300 & 0.262 & 0.060 \\
& CBS $1/\alpha=0.25$ & \textbf{0.480} & 0.238 & 0.106 \\
& CBS $1/\alpha=0.7$ & 0.436 & \textbf{0.329} & \textbf{0.131} \\
\midrule
\multirow{7}{*}{\textbf{Qwen 2.5-7B}}
& Base $\tau=1.0$ & 0.326 & 0.500 & 0.222 \\
& Low Temperature $\tau=0.25$ & 0.508 & 0.677 & 0.288 \\
& Beam Search & 0.570 & 0.707 & 0.247 \\
& Power Sampling $\tau=0.25$ & 0.626 & 0.738 & 0.283 \\
& Power Sampling $\tau=1.0$ & 0.450 & 0.665 & 0.242 \\
& CBS $1/\alpha=0.25$ & 0.610 & \textbf{0.744} & \textbf{0.313} \\
& CBS $1/\alpha=0.7$ & \textbf{0.640} & 0.732 & 0.235 \\
\midrule
\multirow{7}{*}{\textbf{Qwen3.5-2B}}
& Base $\tau=1.0$ & 0.444 & 0.439 & 0.106 \\
& Low Temperature $\tau=0.25$ & 0.632 & 0.600 & 0.217 \\
& Beam Search & 0.656 & 0.560 & 0.177 \\
& Power Sampling $\tau=0.25$ & 0.646 & 0.561 & 0.241 \\
& Power Sampling $\tau=1.0$ & 0.578 & 0.524 & 0.176 \\
& CBS $1/\alpha=0.25$ & 0.672 & \textbf{0.604} & 0.247 \\
& CBS $1/\alpha=0.7$ & \textbf{0.694} & 0.585 & \textbf{0.258} \\
\midrule
\multirow{7}{*}{\begin{tabular}[c]{@{}l@{}}\textbf{Deepseek-Math} \\ \textbf{-7B-Instruct}\end{tabular}}
& Base $\tau=1.0$ & 0.342 & 0.335 & 0.298 \\
& Low Temperature $\tau=0.25$ & 0.436 & 0.445 & 0.298 \\
& Beam Search & 0.432 & 0.494 & 0.293 \\
& Power Sampling $\tau=0.25$ & 0.430 & \textbf{0.537} & 0.278 \\
& Power Sampling $\tau=1.0$ & 0.410 & 0.500 & 0.303 \\
& CBS $1/\alpha=0.25$ & \textbf{0.462} & 0.512 & \textbf{0.369} \\
& CBS $1/\alpha=0.7$ & 0.458 & 0.518 & 0.258 \\
\bottomrule
\end{tabular}
\caption{Performance comparison of CBS across 4 models and 3 benchmarks. 
}
\label{tab:model_eval}
\end{table}
We additionally compare against Power Sampling \cite{ji2026scalablepowersamplingunlocking}, a training-free and verifier-free test-time search method. Across all methods, we use the same task prompts and cap the maximum generation length at $3072$ tokens. For standard beam search, we use a candidate pool of $N=16$ beams and an expansion beam width of $W=4$.

\textbf{Main Results.}
Table~\ref{tab:model_eval} reports pass@1 accuracy on MATH500, HumanEval and GPQA for four models. All comparisons in this work are absolute differences in accuracy, reported in
percentage points.
Across all 12 settings, CBS outperforms both low-temperature sampling and standard beam search, with gains of up to 13.8\% and 12.2\%, respectively. The margin over standard beam search indicates that the improvements cannot be attributed to multi-trajectory search alone. Against power sampling the margin is narrower: CBS wins or tied-best in 11 of the 12 settings, by up to 6.6\%. Overall, CBS is best or tied-best in 11 of the 12 settings, and it leads on MATH500 and GPQA for all four
models. The preferred coefficient is model and task-dependent.

Power sampling relies on a sharpened sampling distribution. Its $\tau=0.25$ variant beats its
$\tau=1.0$ variant in 11 of the 12 settings. CBS improves
accuracy for every model, which we attribute to reweighting tokens by the contrast between the positive and the negative contexts rather than by a single temperature. But, more importantly, CBS has the exact opposite behavior of Power Sampling concerning temperature. It works better with higher temperature, which implies that it is also able to produce more diverse trajectories as we will show in the next pass@$k$ experiment. 

\textbf{Inference Efficiency.}
\begin{figure}[t]
    \centering
    \includegraphics[width=\textwidth]{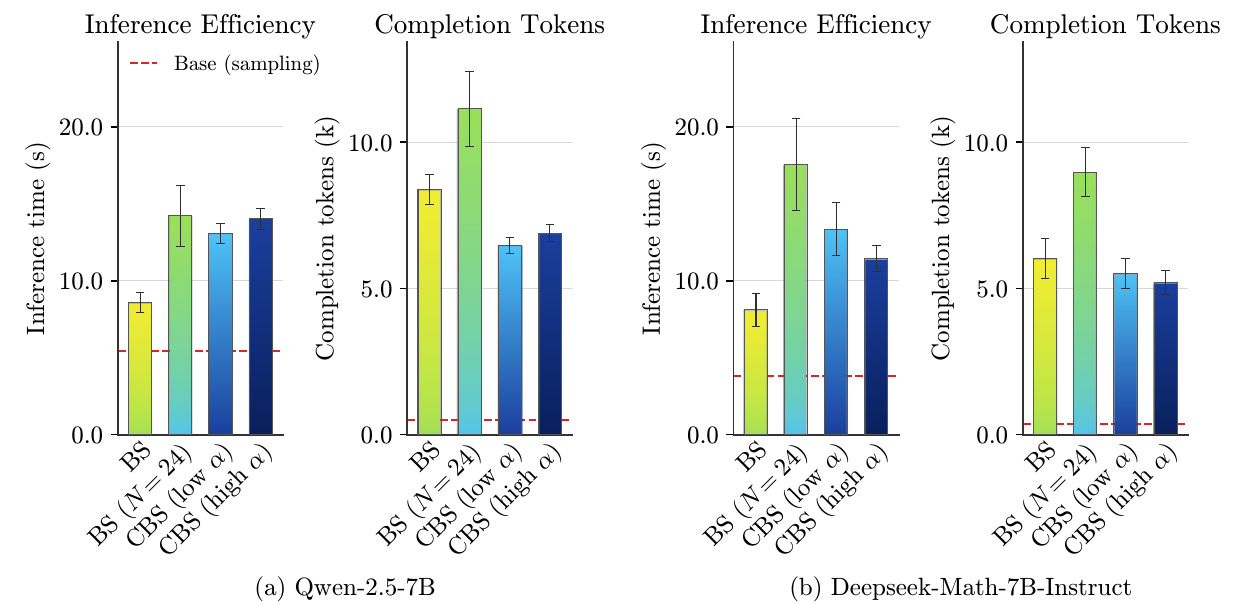}
    \caption{Inference efficiency of (a) Qwen2.5-7B and (b) DeepSeek-Math-7B-Instruct on the MATH500 benchmark.}
    \label{fig:inference_efficiency}
\end{figure}
We next examine the inference cost of CBS.
Figure~\ref{fig:inference_efficiency} reports 
wall-clock inference time and total completion tokens per prompt, counting every candidate generated
\ifarxiv
\begin{figure}[h]
    \centering
    \includegraphics[width=0.3\columnwidth]{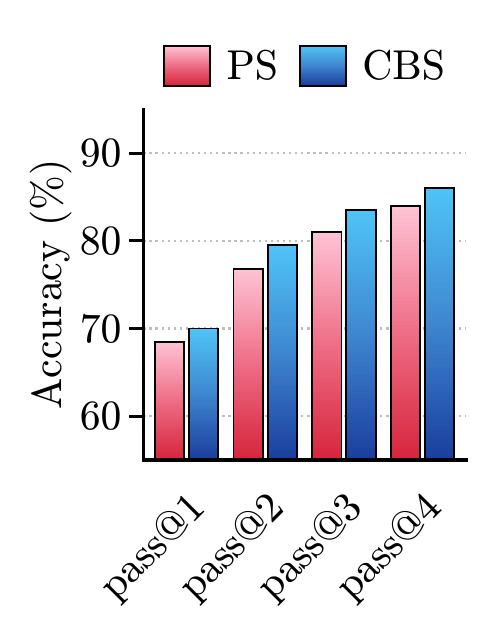}
    \caption{Pass@$k$ on the MATH500 dataset between Power sampling and CBS (Qwen3.5-2B).}
    \label{fig:passk_ps_vs_cbs_qwen}
\end{figure}
\fi
rather than only the selected one, for Qwen2.5-7B and DeepSeek-Math-7B-Instruct on MATH500.
The standard sampling baselines are the fastest.
Standard beam search costs about 2 times more.
CBS is more expensive, requiring roughly $1.5\times$ the cost of standard beam search for both models. 
This overhead does not come from longer outputs. 
CBS returns 20\% fewer completion tokens than beam search for Qwen2.5-7B and 11\% fewer for DeepSeek-Math-7B-Instruct, yet it remains slower than beam search on both models.
The added latency therefore reflects the repeated scoring of each surviving candidate under the positive and negative contexts rather than additional generation.
On Qwen2.5-7B the lower $\alpha$ setting is both faster and more token-efficient, while on DeepSeek-Math-7B-Instruct the higher $\alpha$ wins on both counts.

\textbf{Pass@$\boldsymbol{k}$ Results.} We report pass@$k$ performance by sampling $k$ independent completions on the MATH500 dataset obtained with Qwen3.5-2B in Figure~\ref{fig:passk_ps_vs_cbs_qwen}. Overall, CBS  consistently surpasses power sampling at every evaluated value of $k$. 
While pass@$k$ inherently increases for any method, CBS explicitly steers the decoding process away from common reasoning pitfalls and toward the correct trajectory. Further pass@$k$ results can be found in Appendix \ref{app:passk}.

\subsection{Ablations and Analysis}

\textbf{Effect of Contrastive Context Semantics.}
We conduct a context ablation to determine whether the gains from contrastive scoring arise from the semantic opposition between the positive and negative reasoning contexts, rather than from added context alone. Using Qwen3.5-0.8B, we fix the model and decoding configuration and compare three conditions: no context, the proposed contrastive contexts (excellent versus wrong reasoning), and a neutral control context that assigns responses to two groups without expressing a quality preference. As shown in Table~\ref{tab:cue_ablation}, the contrastive context improves accuracy by 11.1\% and 12.9\% over the empty-context and neutral-context conditions, respectively. The neutral context does not improve over the empty-context condition, suggesting that the gain is not explained by the presence of additional prompt text alone. Instead, these results support the interpretation that the contrastive semantic signal steers the model toward higher-performing responses on MATH500. See Appendix \ref{app:cue_content_ablation} for an extended ablation of various contexts.

\begin{table}[t]
\centering
\small
\renewcommand{\arraystretch}{1.0}
\setlength{\tabcolsep}{1pt}
\renewcommand{\tabularxcolumn}[1]{m{#1}}
\begin{tabularx}{\columnwidth}{@{}m{0.11\columnwidth}>{\raggedright\arraybackslash}Xm{0.10\columnwidth}@{}}
\toprule
\textbf{Context} & \textbf{Content} & \textbf{MATH500} \\
\midrule
No context & - & 0.398 \\
\arrayrulecolor{gray!40}\midrule\arrayrulecolor{black}
Contrastive &
{\scriptsize \texttt{Positive:} \texttt{``This is an example for a response with excellent reasoning:''}\newline
\texttt{Negative:} \texttt{``This is an example for a response with wrong reasoning:''}
}&
\textbf{0.509} \\
\arrayrulecolor{gray!40}\midrule\arrayrulecolor{black}
Neutral &
{\scriptsize \texttt{Positive:} \texttt{``This is an example for a response to the question assigned to control group one:''}\newline
\texttt{Negative:} \texttt{``This is an example for a response to the question assigned to control group two:''} 
}&
0.380 \\
\bottomrule
\end{tabularx}
\caption{Context ablation for Qwen3.5-0.8B on MATH500.
}
\label{tab:cue_ablation}
\end{table}

\textbf{Effect of beam-population budget.}
Because CBS is taking more search time, we increase the standard beam search population from $N=16$ to $N=24$ to test whether allocating more candidate trajectories allows likelihood-based search to overtake CBS with $N=16$. Relative to the mean of the two CBS settings, beam search at N = 24 uses $1.7\times$ as many completion tokens and
takes $1.1\times$ to $1.4\times$ as long on Qwen2.5-7B and DeepSeek-Math-7B-Instruct respectively (Figure~\ref{fig:inference_efficiency}).
\begin{table}[ht]
\centering
\small
\setlength{\tabcolsep}{6pt}
\begin{tabular}{lcccc}
\toprule
\textbf{Method} & \textbf{Qwen3.5-0.8B} & \textbf{Qwen2.5-7B} & \textbf{Qwen3.5-2B} & \textbf{DeepSeek-Math-7B-Instruct} \\
\midrule
Beam search ($N=24$) & 0.391 & 0.600 & 0.676 & 0.452 \\
CBS $1/\alpha=0.25$ & \textbf{0.480} & 0.610 & 0.672 & \textbf{0.462} \\
CBS $1/\alpha=0.7$ & 0.436 & \textbf{0.640} & \textbf{0.694} & 0.458 \\
\bottomrule
\end{tabular}
\caption{Beam-population budget ablation on MATH500.
}
\label{tab:beam_budget_ablation}
\end{table}

\ifarxiv
\begin{figure}[ht]
    \centering
    \includegraphics[width=0.3\columnwidth]
    {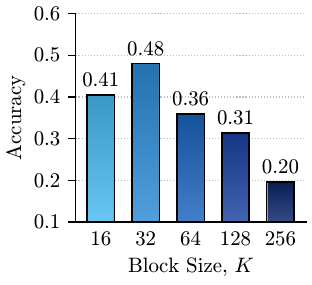}
    \caption{Block size $K$ effect on CBS.}
    \label{fig:block_size_ablation}
\end{figure}
\fi
As shown in Table~\ref{tab:beam_budget_ablation}, for every model, at least one CBS setting outperforms the larger standard beam search baseline.
Thus, increasing the standard beam population can narrow the gap, but does not uniformly recover the gains from CBS, supporting the value of the contrastive ranking signal beyond simply expanding the candidate budget.

\textbf{Effect of beam search block size.}
We vary the block size $K$ to control how often contrastive scoring intervenes during beam search using Qwen3.5-0.8B on MATH500.
As shown in Figure~\ref{fig:block_size_ablation}, accuracy follows an inverted-U pattern: it rises to a peak at $K=32$, and then declines monotonically. 
Large blocks leave few opportunities to redirect unproductive trajectories, while $K=16$ reranks on too little continuation context to score trajectories reliably. Hence, $K=32$ balances these effects.

\section{Related Work}

\textbf{Inference-Time Search.}
Test-time scaling improves language-model reasoning by allocating additional computation during inference rather than changing model parameters \citep{snell2024,welleck2024metageneration,ji2025surelysafealignmentlarge}.
Parallel approaches generate multiple complete reasoning trajectories and aggregate their answers, for example through majority voting in self-consistency \citep{wang2022self} or through verifier and reward-model scores \citep{lightman2023,snell2024}.
Their inference cost grows with the number and length of sampled trajectories, and the most effective allocation of a fixed compute budget depends on problem difficulty \citep{snell2024}.
DeepConf filters low-confidence reasoning traces during or after generation \citep{fu2025deepconf}, while Confidence-Informed Self-Consistency weights sampled answers by model-reported confidence to reduce the number of trajectories needed for aggregation \citep{taubenfeld2025confidence}.
Sequential approaches instead spend inference compute refining a single response, as in Self-Refine's iterative feedback-and-revision loop \citep{madaan2023selfrefine}.
Search-based methods allocate compute within generation: Tree of Thoughts explores explicit intermediate reasoning states \citep{yao2023tree}, ARGS adjusts token probabilities using an external reward signal \citep{khanov2024args}, and TreeBoN branches and prunes partial responses using token-level rewards derived from Direct Preference Optimization \citep{qiu2024treebon}.
Our method belongs to this search-based family, but differs from conventional token-level beam search: it stochastically expands trajectories with multi-token chunks, scores their accumulated responses, and retains a fraction for further expansion.
We therefore use beam-style search as the computational mechanism for allocating test-time compute, rather than treating the search algorithm itself as the contribution.

\textbf{Contrastive and Context-Conditioned Decoding.}
Contrastive Decoding compares an expert and an amateur language model and favors tokens preferred by the expert, subject to a plausibility constraint \citep{li2023contrastive}.
DoLa obtains a contrastive signal within a single model by comparing logits from later and earlier layers \citep{chuang2024dola}, while Asymptotic Probability Decoding interprets and modifies contrastive decoding through extrapolation toward a hypothetical larger model \citep{chang2024explainingimprovingcontrastivedecoding}.
Other methods construct contrasts through the input context: Context-Aware Decoding compares predictions with and without supporting context \citep{shi2023trusting}, and factual-versus-hallucination prompting contrasts output distributions induced by opposing prompts \citep{lv2024improving, yang2024improving}.
Thinking by Subtraction applies contrastive correction selectively at low-confidence positions during reasoning \citep{tang2026thinking}.
These methods primarily intervene in next-token prediction.
In contrast, we use the likelihood difference induced by positive and negative reasoning contexts to score an accumulated trajectory after each sampled chunk.
This contrastive score determines which partial trajectories receive further test-time computation, connecting context-conditioned decoding to the beam-style search procedure above.

\section{Conclusion and Future Work}

We have shown that the self-distillation principle, i.e. using a model's own conditional distributions as a supervisory signal, can be operationalized at test time.
By replacing expert demonstrations with counterfactual contexts, we derived a contrastive PMI reward whose KL-regularized optimum is a Gibbs reweighting of the base policy.
We analyzed that this reweighting is fundamentally global: naive per-token decoding ignores a future correction term.
Contrastive beam search approximates the target distribution by steering generation block by block, and experiments across four models and three benchmarks confirm that it outperforms sampling, beam search, and power sampling baselines on average. Ablations show that the gains arise from the semantic opposition of the contrastive contexts and from intermediate steering during generation, not from increased inference time.

Several directions remain open. The current contrastive contexts are fixed and manually designed; learning or searching for stronger context pairs, potentially in a task-adaptive manner, could improve the steering signal. Reducing the scoring overhead, for example through KV-cache sharing across the three context-conditioned forward passes, would lower the practical cost. More broadly, we view test-time self-distillation as a step toward continual learning agents that leverage their own internal distributions to adapt behavior at inference time: an agent could systematically propose and evaluate counterfactual contexts to guide multi-step planning and action selection without parameter updates.

\bibliography{iclr2027_conference}
\bibliographystyle{iclr2027_conference}

\appendix
\newpage
\section{Proofs}
\label{app:proofs}

Proposition \ref{thm:future_correction} is the contrastive-reward equivalent of a known result on the soft-Bellman view of entropy-regularized reinforcement learning \citep{haarnoja2018softactorcriticoffpolicymaximum,ziebart2008maximum}. 
Our contribution is its application to the contrastive PMI reward Eq.~\ref{eq:reward} and the resulting characterization of how truncating the future correction shapes the approximation error of contrastive beam search.

\begin{proof}[Proof of Proposition \ref{thm:future_correction}]
Fix $t$ and partial sequence $a_{<t}$. The per-token conditional is obtained by marginalizing over future continuations:
\begin{align*}
    \pi_{\mathrm{target}}(a_t \mid q, a_{<t})
    &= \sum_{a_{t+1:T}} \pi_{\mathrm{target}}(a_{t:T} \mid q, a_{<t}) \\
    &= \frac{\sum_{a_{t+1:T}} \prod_{s=t}^{T} \pi_\theta(a_s \mid q, a_{<s}) \exp\!\left(\frac{1}{\alpha} r_s(a_s, q)\right)}{\sum_{a_{t:T}} \prod_{s=t}^{T} \pi_\theta(a_s \mid q, a_{<s}) \exp\!\left(\frac{1}{\alpha} r_s(a_s, q)\right)}.
\end{align*}
Factoring the numerator by separating position $t$ from future terms $s > t$:
\begin{align*}
    &\sum_{a_{t+1:T}} \prod_{s=t}^{T} \pi_\theta(a_s \mid q, a_{<s}) \exp\!\left(\frac{1}{\alpha} r_s(a_s, q)\right)
    = \\
    &\pi_\theta(a_t \mid q, a_{<t}) \exp\!\left(\frac{1}{\alpha} r_t(a_t, q)\right) \cdot \underbrace{\sum_{a_{t+1:T}} \prod_{s=t+1}^{T} \pi_\theta(a_s \mid q, a_{<t}, a_t, a_{<s}) \exp\!\left(\frac{1}{\alpha} r_s(a_s, q)\right)}_{\zeta_t(a_t, q)}.
\end{align*}
Applying the same factorization to each $a'_t$ in the denominator yields Eq.~\ref{eq:per_token_decomp}. For the final token ($t = T$), the future correction reduces to $\zeta_T = 1$ (empty product), recovering the naive per-token distribution.
\end{proof}

\begin{lemma}[TV Distance under Bounded Reweighting]
\label{lem:tv_reweight}
Let $p(a) = \frac{w(a)}{\sum_{a'} w(a')}$ and $q(a) = \frac{w(a) f(a)}{\sum_{a'} w(a') f(a')}$ for positive weights $w(a) > 0$ and bounded function $f(a) \in [c, C]$ with $0 < c \leq C < \infty$. Then
\begin{equation*}
    \|p - q\|_{\mathrm{TV}} \leq \frac{C - c}{C + c}.
\end{equation*}
\end{lemma}

This is the classical bounded-likelihood-ratio bound on total variation; we restate it for completeness \citep{tsybakov2009introduction}.

\begin{proof}
For any event $A$:
\begin{align*}
    |p(A) - q(A)| &= \left|\sum_{a \in A} w(a) \left(\frac{1}{Z_w} - \frac{f(a)}{Z_{wf}}\right)\right| = \left|\sum_{a \in A} \frac{w(a)}{Z_w}\left(1 - \frac{f(a) Z_w}{Z_{wf}}\right)\right|.
\end{align*}
Since $c \leq f(a) \leq C$, we have $\frac{c \, Z_w}{Z_{wf}} \leq \frac{f(a) Z_w}{Z_{wf}} \leq \frac{C \, Z_w}{Z_{wf}}$. Noting that $\frac{Z_w}{Z_{wf}} \in [\frac{1}{C}, \frac{1}{c}]$ (because $c \, Z_w \leq Z_{wf} \leq C \, Z_w$), the ratio $\frac{f(a) Z_w}{Z_{wf}}$ ranges in $[\frac{c}{C}, \frac{C}{c}]$. Using the equivalent form for total variation distance:
\begin{align}\label{Eq_interm}
    &||p - q||_{\mathrm{TV}} = \frac{1}{2}\sum_{a}|p(a) - q(a)| = \frac{1}{2}\sum_{a}p(a)\left|1 - \frac{q(a)}{p(a)}\right| = \frac{1}{2}\sum_{a}p(a)\left|1 - \frac{f(a)Z_w}{Z_{wf}}\right|
\end{align}
Let us study the expression $\left|1 - \frac{f(a)Z_w}{Z_{wf}}\right|$ with more care. Let $g(r) = \frac{|1-r|}{1+r}$ and parameter $r$ belong to the interval $r\in \left[\frac{c}{C}, \frac{C}{c}\right]$. We can make the following two observations:
\begin{enumerate}
    \item If $r \ge 1$, then $g(r) = \frac{r-1}{r+1}$ and $g'(r) = \frac{2}{(r+1)^2} > 0$ for any $r$. Hence, for $r\in [1,\frac{C}{c}]$ due to monotonic increment of function $g(r)$ we can write:
    \begin{align*}
        g(r) \le g\left(\frac{C}{c}\right) = \frac{\frac{C}{c} - 1}{\frac{C}{c} + 1} = \frac{C-c}{C+c}
    \end{align*}
    \item If $r<1$, then $g(r) = \frac{1-r}{1+r}$ and $g'(r) = -\frac{2}{(r+1)^2} < 0$ for any $r$. Hence, for $r\in [\frac{c}{C},1)$ due to monotonic decrement of function $g(r)$ we can write:
    \begin{align*}
        g(r) \le g\left(\frac{c}{C}\right) = \frac{1 - \frac{c}{C}}{\frac{c}{C} + 1} = \frac{C-c}{C+c}
    \end{align*}
\end{enumerate}
Because $\frac{f(a)Z_w}{Z_{wf}}\in \left[\frac{c}{C}, \frac{C}{c}\right]$ we have from the above results:
\begin{align*}
    \left|1 - \frac{f(a)Z_w}{Z_{wf}}\right| \le \frac{C-c}{C+c}\left[ 1 + \frac{f(a)Z_w}{Z_{wf}}\right]
\end{align*}
Applying this result in Eq \ref{Eq_interm} gives:
\begin{align*}
    &||p - q||_{\mathrm{TV}} = \frac{1}{2}\sum_{a}p(a)\left|1 - \frac{f(a)Z_w}{Z_{wf}}\right| \le \frac{1}{2}\frac{C-c}{C+c}\sum_{a}p(a)\left[1 + \frac{f(a)Z_w}{Z_{wf}}\right] = \\\nonumber
    &\frac{1}{2}\frac{C-c}{C+c}\left[\sum_ap(a) + \sum_a\frac{w(a)}{Z_w}\frac{f(a)Z_w}{Z_{wf}}\right] = \frac{1}{2}\frac{C-c}{C+c}\left[1 + \sum_{a}\frac{f(a)w(a)}{Z_{wf}}\right] = \frac{2}{2}\frac{C-c}{C+c} = \frac{C-c}{C+c}
\end{align*}
\end{proof}

\begin{proof}[Proof of Theorem \ref{thm:truncation_bias}]
\textbf{Part (i).}
At a scoring point at position $t$ (start of a chunk), the $K$ tokens within the chunk (positions $t{+}1$ through $t{+}K$) are scored exactly as part of the accumulated reward. The future correction accounts for the remaining $T - t - K$ tokens after the chunk:
\begin{equation*}
    \zeta_t(a_t, q) = \mathbb{E}_{a_{t+K+1:T} \sim \pi_\theta}\!\left[\prod_{s>t+K} \exp(r_s/\alpha)\right].
\end{equation*}
Since $|r_s| \leq r_{\max}$ and there are $T - t - K$ future tokens:
\begin{equation*}
    e^{-(T-t-K) r_{\max}/\alpha} \leq \prod_{s>t+K} e^{r_s/\alpha} \leq e^{(T-t-K) r_{\max}/\alpha} \quad \Longrightarrow \quad \zeta_t(a_t, q) \in \bigl[e^{-(T-t-K) r_{\max}/\alpha},\; e^{(T-t-K) r_{\max}/\alpha}\bigr].
\end{equation*}
The truncation estimator sets $\hat{\zeta}_t \equiv 1$. Applying Lemma~\ref{lem:tv_reweight} with $c = e^{-(T-t-K)r_{\max}/\alpha}$, $C = e^{(T-t-K)r_{\max}/\alpha}$, and $f = \zeta_t$:
\begin{equation*}
    \|\tilde{\pi}_{\mathrm{trunc}} - \pi_{\mathrm{target}}\|_{\mathrm{TV}} \leq \frac{e^{(T-t-K)r_{\max}/\alpha} - e^{-(T-t-K)r_{\max}/\alpha}}{e^{(T-t-K)r_{\max}/\alpha} + e^{-(T-t-K)r_{\max}/\alpha}} = \tanh\!\left(\frac{(T-t-K)\,r_{\max}}{\alpha}\right).
\end{equation*}
Scoring decisions occur at chunk boundaries $t = 0, K, 2K, \ldots, T{-}K$, giving $T/K$ decision points. By the standard Markov chain perturbation bound \citep{levin2026markov}, the trajectory-level error compounds:
\begin{align*}
    \|\tilde{\pi}_{\mathrm{traj}} - \pi_{\mathrm{target}}\|_{\mathrm{TV}} \leq \sum_{j=0}^{T/K - 1} \tanh\!\left(\frac{(T - jK - K)\,r_{\max}}{\alpha}\right) = \\
    \sum_{m=1}^{T/K} \tanh\!\left(\frac{(T - mK)\,r_{\max}}{\alpha}\right) \leq \frac{r_{\max}}{\alpha}\frac{T(T-K)}{2K},
\end{align*}
using $\tanh(x) \leq x$ and the identity $\sum_{m=1}^{n}(T - mK) = K \cdot \frac{n(n-1)}{2} = \frac{T(T-K)}{2K}$ with $n = T/K$. For comparison, naive per-token sampling (Eq.~\ref{eq:naive}) truncates at every token with horizon $T - t$, giving the looser bound $\frac{r_{\max}}{\alpha}\frac{T(T-1)}{2}$.

\smallskip
\textbf{Part (ii).}
Under the zero-mean condition $\mathbb{E}_{a_t \sim \pi_\theta(\cdot \mid q, a_{<t})}[r_t(a_t, q)] = 0$, the variables $\{r_s(a_s, q)/\alpha\}_{s>t+K}$ are bounded martingale differences: $|r_s/\alpha| \leq r_{\max}/\alpha$ and $\mathbb{E}[r_s/\alpha \mid q, a_{<s}] = 0$. The future correction is the moment-generating function of their sum:
\begin{equation*}
    \zeta_t(a_t, q) = \mathbb{E}\!\left[\exp\!\left(\sum_{s>t+K} \frac{r_s}{\alpha}\right)\right].
\end{equation*}
By Hoeffding's lemma applied conditionally to each martingale difference and iterated via the tower property:
\begin{equation*}
    \zeta_t(a_t, q) \leq \exp\!\left(\frac{(T-t-K)\, r_{\max}^2}{2\alpha^2}\right).
\end{equation*}
By Jensen's inequality, since $\mathbb{E}\bigl[\sum_{s>t+K} r_s/\alpha \,\big|\, q, a_{<t}, a_t\bigr] = 0$:
\begin{equation*}
    \zeta_t(a_t, q) \geq \exp\!\left(\mathbb{E}\!\left[\sum_{s>t+K} \frac{r_s}{\alpha}\right]\right) = 1.
\end{equation*}
Applying Lemma~\ref{lem:tv_reweight} with $c = 1$ and $C = \exp\!\bigl((T-t-K)\,r_{\max}^2 / (2\alpha^2)\bigr)$:
\begin{equation*}
    \|\tilde{\pi}_{\mathrm{trunc}} - \pi_{\mathrm{target}}\|_{\mathrm{TV}} \leq \frac{e^{(T-t-K)r_{\max}^2/(2\alpha^2)} - 1}{e^{(T-t-K)r_{\max}^2/(2\alpha^2)} + 1} = \tanh\!\left(\frac{(T-t-K)\, r_{\max}^2}{4\alpha^2}\right).
\end{equation*}
The trajectory bound follows identically to part~(i), summing over $T/K$ scoring points:
\begin{equation*}
    \sum_{m=1}^{T/K} \tanh\!\left(\frac{(T - mK)\,r_{\max}^2}{4\alpha^2}\right) \leq \frac{r_{\max}^2}{4\alpha^2}\frac{T(T-K)}{2K}. \qedhere
\end{equation*}
\end{proof}

\section{Experimental Setups and Implementation Details}
\label{app:experimental_details}

\textbf{Prompts.}
Table~\ref{tab:task_prompts} lists the task prompt templates used for MATH500, HumanEval, and GPQA. All decoding methods within a model--benchmark setting use the same task prompt, so comparisons vary only the decoding and candidate-ranking procedure. For contrastive scoring, the positive or negative context is appended to the user message after the original query and immediately before the assistant-generation marker. Denoting the system instruction by $s$, the user query by $q$, and either context by $c^{\pm}$, the context-conditioned serialization is
\begin{equation}
    [\textsc{System}:s][\textsc{User}:q \oplus c^{\pm}][\textsc{Assistant}:],
    \label{eq:cue_serialization}
\end{equation}
where $\oplus$ denotes textual concatenation. The system instruction is unchanged across the base, positive, and negative evaluations. The base context contains only $q$ in the user message, while the context-conditioned prompts append $c^+$ or $c^-$, respectively. These contexts affect candidate scoring but are not included in the returned response. Table~\ref{tab:qwen_prompt_example} gives a complete positive-context example using the Qwen chat format; the negative context is formed by replacing the final positive context with ``This is an example for a response with wrong reasoning:''.

\textbf{Hyperparameter settings.}
Table~\ref{tab:decoding_hyperparameters} summarizes the principal settings used in the main evaluation. All methods use a block size of $K=32$, 96 generation iterations, and a maximum generation length of $T_{\max}=3072$ tokens, with early termination upon emitting an EOS token. Standard and low-temperature sampling use a single trajectory with $N=W=1$ and disable contrastive scoring. Standard sampling, beam search, and both contrastive beam search variants use a generation temperature of $\tau=1.0$, while the low-temperature baseline uses $\tau=0.25$. Standard and contrastive beam search use an active population of $N=16$ and pruning factor $W=4$, retaining the top $N/W=4$ unfinished trajectories after each scoring stage. These settings are used throughout the main experiments unless explicitly varied in an ablation. The baseline ranks candidates using only the mean base-context log-likelihood ($1/\alpha=0$), whereas the contrastive variants use $1/\alpha\in\{0.25,0.7\}$. To ensure numerical stability, all likelihood computations are carried out in log-space.

\textbf{Inference efficiency measurement.}
We measure all inference efficiency in Figure~\ref{fig:inference_efficiency} on a
single accelerator with the same runtime stack.
Each time is the wall-clock time per prompt for generation and candidate
scoring, and excludes model loading, dataset loading, and metric computation. We
count completion tokens as all generated tokens per prompt across every
candidate, not just the final answer, and exclude prompt tokens and the scoring
passes.

\begin{table*}[t]
	\centering
	\begin{tabularx}{\textwidth}{>{\hsize=0.2\hsize\bfseries}X >{\hsize=0.8\hsize}X}
		\toprule
		Task & Prompt Template \\
		\midrule
		
		MATH500 & 
		\textbf{System:} You are a helpful AI Assistant that provides well-reasoned and detailed responses. You first think about the reasoning process as an internal monologue and then provide the user with the boxed answer. Respond in the following format: \textless{}think\textgreater{} ... \textless{}/think\textgreater{} \textless{}answer\textgreater{} \textbackslash{}boxed\{...\} \textless{}/answer\textgreater{}. \newline \newline
		\textbf{User:} \{problem\} \\
		\midrule
		
		HumanEval & 
		\textbf{System:} Please reason step by step internally. Then output ONLY the final Python code that completes the task within triple backticks. Do not include explanations, markdown, or \textbackslash{}boxed\{\}. \newline \newline
		\textbf{User:} Complete the following Python function:\newline
		\{prompt\} \\
		\midrule
		
		GPQA & 
		\textbf{System:} You are a helpful AI Assistant. Please reason step by step, and put your final answer within \textbackslash{}boxed\{\}. \newline \newline
		\textbf{User:} Answer the following multiple choice question. The last line of your response should be of the following format: `\textbackslash{}boxed\{\$LETTER\}' (without quotes) where LETTER is one of ABCD (ex. `\textbackslash{}boxed\{A\}'). Think step by step before answering. \newline \newline
		\{question\} \newline \newline
		A) \{choice\_A\} \quad B) \{choice\_B\} \quad C) \{choice\_C\} \quad D) \{choice\_D\} \\
		
		\bottomrule
	\end{tabularx}
	\caption{Task Prompt Templates}
	\label{tab:task_prompts}
\end{table*}

\begin{table*}[t]
\centering
\small
\begin{tabularx}{\textwidth}{lX}
\toprule
\textbf{Role} & \textbf{Serialized Qwen prompt content} \\
\midrule
System & \texttt{<|im\_start|>system} You are a helpful AI Assistant that provides well-reasoned and detailed responses. You first think about the reasoning process as an internal monologue and then provide the user with the boxed answer. Respond in the following format: \textless{}think\textgreater{} ... \textless{}/think\textgreater{} \textless{}answer\textgreater{} \textbackslash{}boxed\{...\} \textless{}/answer\textgreater{}. \texttt{<|im\_end|>} \\
User query and positive context & \texttt{<|im\_start|>user} Solve for $x$: $2^{x+1}=32$. This is an example for a response with excellent reasoning: \texttt{<|im\_end|>} \\
Assistant prefix & \texttt{<|im\_start|>assistant} \\
\bottomrule
\end{tabularx}
\caption{Example serialization of a MATH500-style query with the positive contrastive context under the Qwen chat template.}
\label{tab:qwen_prompt_example}
\end{table*}

\begin{table*}[t]
\centering
\small
\begin{tabularx}{\textwidth}{lX}
\toprule
\textbf{Method} & \textbf{Hyperparameters} \\
\midrule
Standard sampling & $N=1$, $W=1$, $K=32$, $\tau=1.0$ \\
Low-temperature sampling & $N=1$, $W=1$, $K=32$, $\tau=0.25$ \\
Beam search & $N=16$, $W=4$, $K=32$, $\tau=1.0$, $1/\alpha=0$ \\
Contrastive beam search (low $\alpha$) & $N=16$, $W=4$, $K=32$, $\tau=1.0$, $1/\alpha=0.25$ \\
Contrastive beam search (high $\alpha$) & $N=16$, $W=4$, $K=32$, $\tau=1.0$, $1/\alpha=0.7$ \\
Power Sampling & $K_t=16$, $M_t=16$, $B=32$ (equivalent to $K$), $\tau=0.25$ (low) and $\tau=1.0$ (high), $\alpha=4$ \\
Best-of-$N$ & $N=16$, $\tau=1.0$ \\
\bottomrule
\end{tabularx}
\caption{Main hyperparameter settings for each method.}
\label{tab:decoding_hyperparameters}
\end{table*}

\clearpage
\section{Further Experimental Results}

\subsection{Additional Pass@${k}$ results and analysis.}
\label{app:passk}
Figure~\ref{fig:passk_ps_vs_cbs_deepseek} reports pass@$k$ on MATH500 for DeepSeek-Math-7B-Instruct. Standard beam search marginally surpasses power sampling at $k=1$ and $k=2$ but yields diminishing returns, falling behind at $k=3$ and $k=4$. In contrast, CBS and power sampling scale consistently, with CBS outperforming both methods across all $k$. Because both beam search variants use a temperature of $\tau=1.0$ to maintain generation diversity, this divergence demonstrates that unguided likelihood maximization is insufficient for complex reasoning. Instead, the contrastive contexts in CBS effectively steer the search trajectory toward correct solutions.

\begin{figure}[h]
    \centering
    \includegraphics[width=0.4\columnwidth]{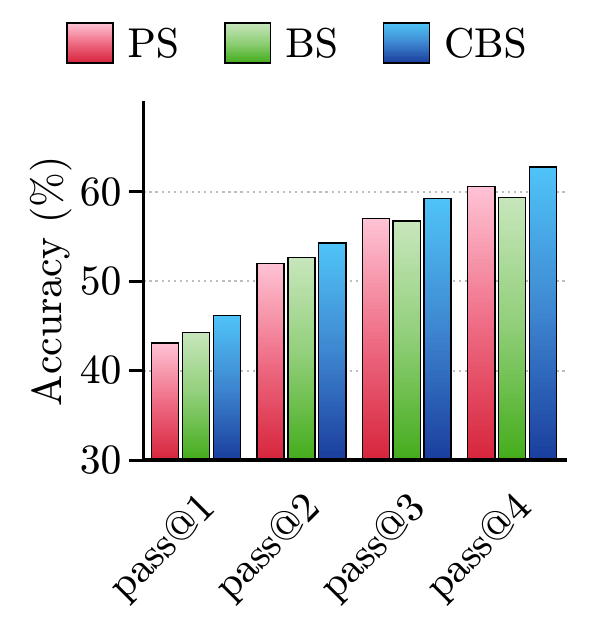}
    \caption{Pass@$k$ performance $k \in \{1, 2, 3, 4\}$ on MATH500 dataset for Power sampling (PS), standard beam search (BS) and contrastive beam search (CBS) using Deepseek-Math-7B-Instruct.}
    \label{fig:passk_ps_vs_cbs_deepseek}
\end{figure}

\subsection{Context content ablation}
\label{app:cue_content_ablation}

 Having established that performance gains arise from the contrastive semantic signal, we further explore how context content impacts performance. We generate 10 contrastive context pairs with varying contexts (Table~\ref{tab:cue-pairs}). Each positive contexts is paired with its counterfactual negative (e.g., "excellent reasoning" vs. "wrong reasoning" or "complete" vs. "incomplete" treatment). We use Best-of-$N$ to generate $N = 16$ candidates, recompute log probabilities under the positive and negative prefixes, and rank them by ${lp}_{\text{base}} + \frac{1}{\alpha}\cdot({lp}_{+} - {lp}_{-})$. We test this across three benchmarks with distinct demands: MATH500 (arithmetic reasoning), HumanEval (functional correctness in coding), and GPQA (expert-level science).

 As shown in Figure~\ref{fig:contrastive-cue-contexts}, on the MATH500 dataset, the reasoning and step verification contexts yield the highest accuracy. This aligns with how math models usually fail, where one incorrect step breaks the whole solution, making these specific contexts highly effective. The remaining contexts perform near the baseline, with coherence and self-correction ranking lowest. Conversely, HumanEval task exhibits minimal contexts differentiation. Most contexts, including reasoning, completeness, reliability, and logical validity, cluster at similar accuracy levels, while coherence and decomposition fall marginally below the baseline. Because the reranker evaluates text rather than executing the code, modifying the evaluation criteria provides no leverage for detecting hidden runtime errors. Similar to MATH500, reasoning remains the strongest context on GPQA, but the subsequent rankings shift to favor reliability and logical validity. This benchmark evaluates complex scientific questions that depend on strict factual consistency and sound argumentation. This explains why logical validity and reliability provide an advantage while structural contexts like coherence do not. Although contrastive contexts yield only marginal accuracy gains during best-of-N terminal reranking, our results demonstrate that applying them during intermediate steering produces substantial performance improvements.
 
 Overall, the generic reasoning context consistently performs best across all three benchmarks. However, performance of secondary contexts is domain-dependent, such as step verification for math and logical validity for science, demonstrating that the nature of the task determines which context works best.

  \begin{figure}[htbp]
    \centering
    
    \begin{subfigure}[b]{0.77\textwidth}
        \centering
        \includegraphics[width=1\columnwidth]{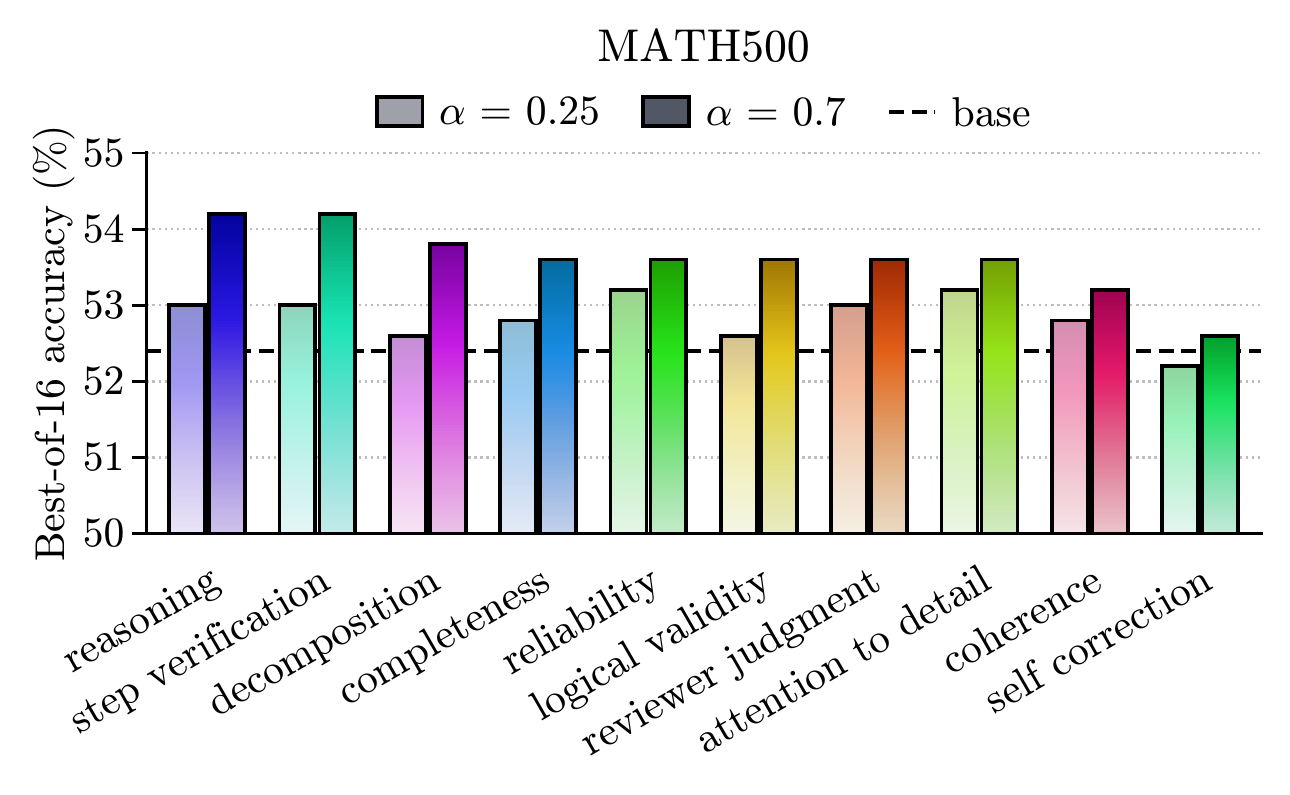}
        \label{fig:sub-a}
    \end{subfigure}
    \vspace{0.3cm} 
    \begin{subfigure}[b]{0.77\textwidth}
        \centering
        \includegraphics[width=1\columnwidth]{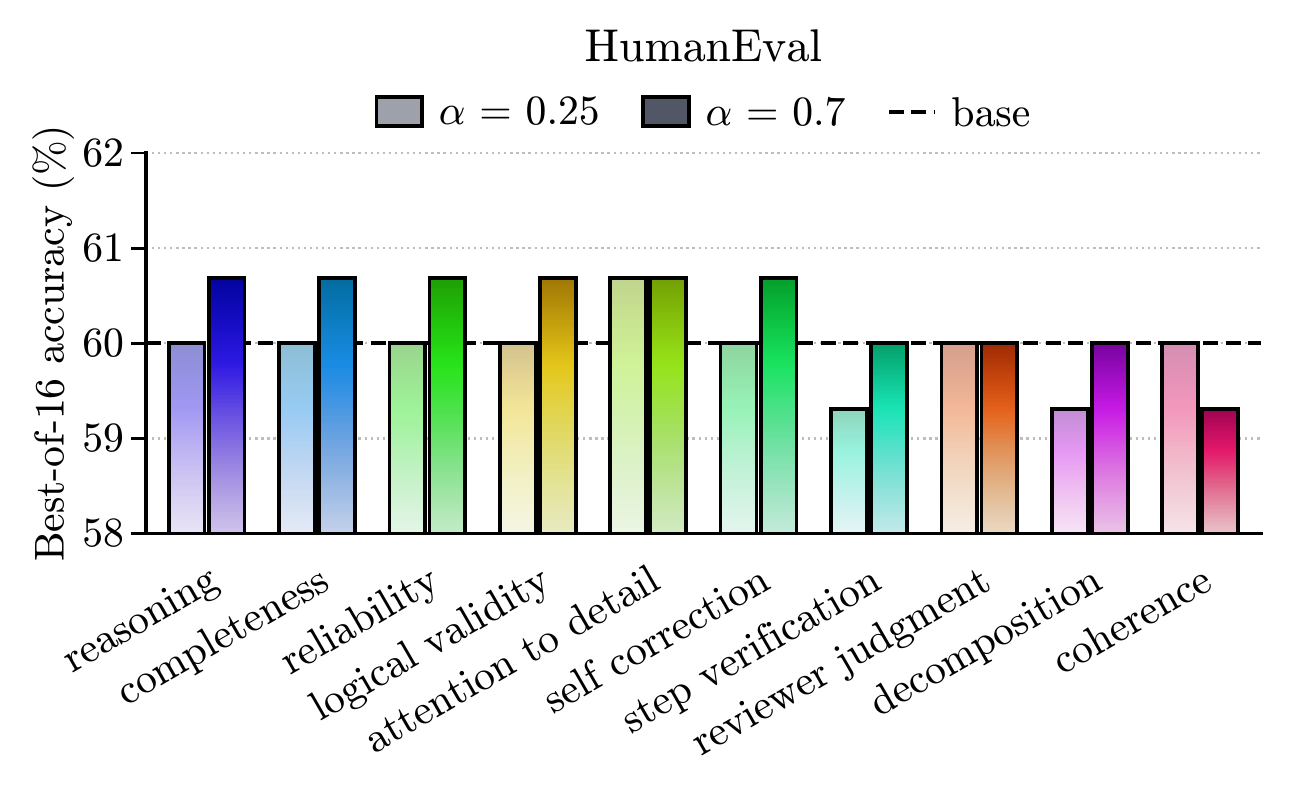}
        \label{fig:sub-b}
    \end{subfigure}
    \vspace{0.3cm}
    \begin{subfigure}[b]{0.77\textwidth}
        \centering
        \includegraphics[width=1\columnwidth]{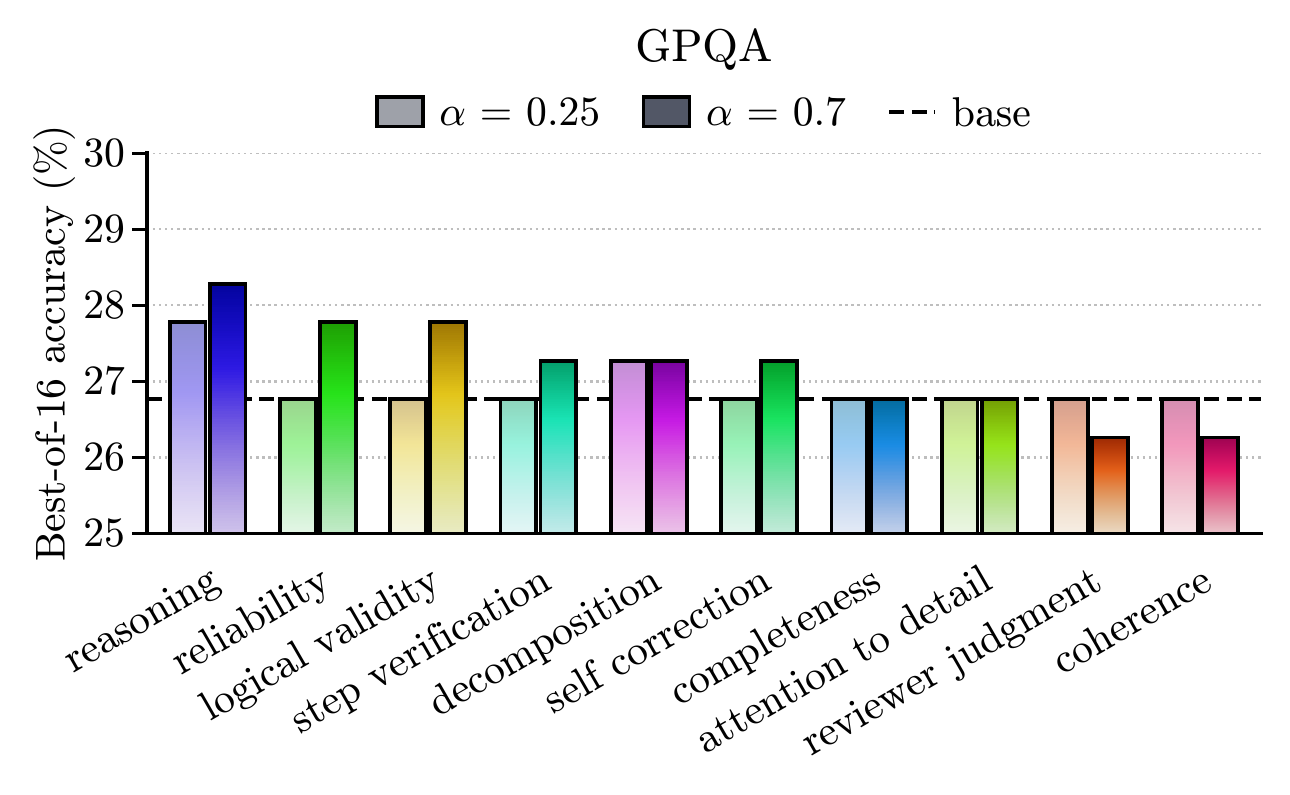}
        \label{fig:sub-c}
    \end{subfigure}
    
    \caption{Best-of-16 accuracy for ten contrastive context pairs on MATH500, HumanEval and GPQA (Qwen2.5-7B), at the two deployed steering strengths. Dashed line: no-context baseline.}
    \label{fig:contrastive-cue-contexts}
\end{figure}

\begin{table}[t]
\centering
\small
\renewcommand{\arraystretch}{1.0}
\setlength{\tabcolsep}{1pt}
\renewcommand{\tabularxcolumn}[1]{m{#1}}
\begin{tabularx}{\columnwidth}{@{} >{\raggedright\arraybackslash}m{0.16\columnwidth} >{\raggedright\arraybackslash}X @{}}
\toprule
\textbf{Content} & \textbf{Contrastive Context Pair} \\
\midrule
Reasoning &
{\scriptsize \texttt{Positive:} \texttt{``This is an example for a response with excellent reasoning:''}\newline
\texttt{Negative:} \texttt{``This is an example for a response with wrong reasoning:''}
} \\
\arrayrulecolor{gray!40}\midrule\arrayrulecolor{black}
Completeness &
{\scriptsize \texttt{Positive:} \texttt{``This is an example for a response with a complete treatment of the problem:''}\newline
\texttt{Negative:} \texttt{``This is an example for a response with an incomplete treatment of the problem:''}
} \\
\arrayrulecolor{gray!40}\midrule\arrayrulecolor{black}
Step verification &
{\scriptsize \texttt{Positive:} \texttt{``This is an example for a response that carefully verifies each step of its reasoning:''}\newline
\texttt{Negative:} \texttt{``This is an example for a response that skips verification and rushes to a conclusion:''}
} \\
\arrayrulecolor{gray!40}\midrule\arrayrulecolor{black}
Reliability &
{\scriptsize \texttt{Positive:} \texttt{``This is an example for a response that is reliable and trustworthy:''}\newline
\texttt{Negative:} \texttt{``This is an example for a response that is unreliable and error-prone:''}
} \\
\arrayrulecolor{gray!40}\midrule\arrayrulecolor{black}
Logical validity &
{\scriptsize \texttt{Positive:} \texttt{``This is an example for a response with logically sound arguments:''}\newline
\texttt{Negative:} \texttt{``This is an example for a response with logical fallacies:''}
} \\
\arrayrulecolor{gray!40}\midrule\arrayrulecolor{black}
Reviewer judgment &
{\scriptsize \texttt{Positive:} \texttt{``This is an example for a response that a careful reviewer would approve without hesitation:''}\newline
\texttt{Negative:} \texttt{``This is an example for a response that a careful reviewer would reject immediately:''}
} \\
\arrayrulecolor{gray!40}\midrule\arrayrulecolor{black}
Coherence &
{\scriptsize \texttt{Positive:} \texttt{``This is an example for a response that is internally consistent from start to finish:''}\newline
\texttt{Negative:} \texttt{``This is an example for a response that contradicts itself partway through:''}
} \\
\arrayrulecolor{gray!40}\midrule\arrayrulecolor{black}
Decomposition &
{\scriptsize \texttt{Positive:} \texttt{``This is an example for a response that breaks the problem into clear, manageable steps:''}\newline
\texttt{Negative:} \texttt{``This is an example for a response that attempts the problem in one confused leap:''}
} \\
\arrayrulecolor{gray!40}\midrule\arrayrulecolor{black}
Attention to detail &
{\scriptsize \texttt{Positive:} \texttt{``This is an example for a response that pays careful attention to every detail:''}\newline
\texttt{Negative:} \texttt{``This is an example for a response that overlooks important details:''}
} \\
\arrayrulecolor{gray!40}\midrule\arrayrulecolor{black}
Self-correction &
{\scriptsize \texttt{Positive:} \texttt{``This is an example for a response that catches and corrects its own mistakes:''}\newline
\texttt{Negative:} \texttt{``This is an example for a response that persists in its mistakes without noticing them:''}
} \\
\bottomrule
\end{tabularx}
\caption{Examples of 10 contrastive context pairs. Every pair shares same initial framing with
``This is an example for a response \ldots{}''. Each positive context is
paired with its counterfactual negative.}
\label{tab:cue-pairs}
\end{table}

\end{document}